\documentclass[fleqn]{llncs}
\pdfoutput=1 
\usepackage{url}
\usepackage{amssymb} 

\newcommand{\etal}{{\it et al.}}
\newcommand{\SB}{\textsf{sb}}
\newcommand\true{\mathit{true}}
\newcommand\false{\mathit{false}}
\newcommand{\bee}{\textsf{BEE}}
\newcommand{\bz}{{\bf 0}} 
\newcommand{\ba}{{\bf 1}} 
\newcommand{\bb}{{\bf 2}} 
\newcommand{\bc}{{\bf 3}}  
\newcommand{\set}[1]{\left\{
      \begin{array}{l}#1\end{array}
      \right\}}
\newcommand{\sset}[2]{\left\{~#1  \left|
      \begin{array}{l}#2\end{array}
    \right.     \right\}}
\newcommand\tuple[1]{\langle #1 \rangle}

\renewcommand{\AA}{{\cal A}}

\newcommand{\RR}{{\cal R}}
\newcommand{\MM}{{\cal M}}

\newtheorem{observation}{Observation}
\usepackage{mathtools}
\usepackage{amssymb} 

\usepackage{graphics}
\usepackage{multirow}
\usepackage{amsmath}
\usepackage{wrapfig}
\usepackage{multirow}
\usepackage{color}
\usepackage{xcolor}

\title{Computing the Ramsey Number R(4,3,3) using Abstraction and
  Symmetry breaking\thanks{Supported by the Israel Science
    Foundation, grant 182/13.}}

\author{Michael Codish\inst{1} \and Michael Frank\inst{1} \and 
        Avraham Itzhakov\inst{1}\and Alice Miller\inst{2}}
\institute{
  Department of Computer Science,
  Ben-Gurion University of the Negev, Israel
\and
  School of Computing Science,
  University of Glasgow, Scotland
}

\begin{document}
\maketitle
\begin{abstract}
  The number $R(4,3,3)$ is often presented as the unknown Ramsey
  number with the best chances of being found ``soon''. Yet, its
  precise value has remained unknown for almost 50 years.
  This paper presents a methodology based on
  \emph{abstraction} and \emph{symmetry breaking} that applies to
  solve hard graph edge-coloring problems.
  The utility of this methodology is demonstrated by using it to
  compute the value $R(4,3,3)=30$. Along the way it is required to
  first compute the previously unknown set $\RR(3,3,3;13)$ consisting
  of 78{,}892 Ramsey colorings.

\end{abstract}

\section{Introduction}\label{sec:intro}

This paper introduces a general methodology that applies to solve
graph edge-coloring problems and demonstrates its application in the
search for Ramsey numbers. These are notoriously hard graph coloring
problems that involve assigning colors to the edges of a complete
graph. 
An $(r_1,\ldots,r_k;n)$ Ramsey coloring is a graph coloring in $k$
colors of the complete graph $K_n$ that does not contain a
monochromatic complete sub-graph $K_{r_i}$ in color $i$ for each
$1\leq i\leq k$. The set of all such colorings is denoted
$\RR(r_1,\ldots,r_k;n)$.
The Ramsey number $R(r_1,\ldots,r_k)$ is the least $n>0$ such that no
$(r_1,\ldots,r_k;n)$ coloring exists.
In particular, the number $R(4,3,3)$ is often presented as the unknown
Ramsey number with the best chances of being found ``soon''. Yet, its
precise value has remained unknown for more than 50 years.  It is
currently known that $30\leq R(4,3,3)\leq 31$.
Kalbfleisch~\cite{kalb66} proved in 1966 that $R(4,3,3)\geq 30$,
Piwakowski~\cite{Piwakowski97} proved in 1997 that $R(4,3,3)\leq 32$,
and one year later Piwakowski and Radziszowski~\cite{PR98} proved that
$R(4,3,3)\leq 31$. We demonstrate how our methodology applies to
computationally prove that $R(4,3,3)=30$.

Our strategy to compute $R(4,3,3)$ is based on the search for a
$(4,3,3;30)$ Ramsey coloring.  If one exists, then because
$R(4,3,3)\leq 31$, it follows that $R(4,3,3) = 31$. Otherwise, 
because $R(4,3,3)\geq 30$, it follows that $R(4,3,3) = 30$.

In recent years, Boolean SAT solving techniques have improved
dramatically. Today's SAT solvers are considerably faster and able to
manage larger instances than were previously possible. Moreover, encoding and
modeling techniques are better understood and increasingly
innovative. SAT is currently applied to solve a wide variety of hard
and practical combinatorial problems, often outperforming
dedicated algorithms.
The general idea is to encode a (typically, NP) hard problem instance,
$\mu$, to a Boolean formula, $\varphi_\mu$, such that the satisfying
assignments of $\varphi_\mu$ correspond to the solutions of
$\mu$. Given such an encoding, a SAT solver can be applied to solve
$\mu$.

Our methodology in this paper combines SAT solving with two additional
concepts: \emph{abstraction} and \emph{symmetry breaking}.
The paper is structured to let the application drive the presentation
of the methodology in three steps. 
Section~\ref{sec:prelim} presents: preliminaries on graph coloring
problems, some general notation on graphs, and a simple constraint
model for Ramsey coloring problems.
%
Section~\ref{sec:embed} presents the first step in our quest to
compute $R(4,3,3)$. We introduce a basic SAT encoding and detail how a
SAT solver is applied to search for Ramsey colorings. Then we describe
and apply a well known embedding technique, which allows to determine
a set of partial solutions in the search for a $(4,3,3;30)$ Ramsey
coloring such that if a coloring exists then it is an extension of one
of these partial solutions.  This may be viewed as a preprocessing step for
a SAT solver which then starts from a partial solution.  Applying this
technique we conclude that if a $(4,3,3;30)$ Ramsey coloring exists
then it must be $\tuple{13,8,8}$ regular. Namely, each vertex in the
coloring must have 13 edges in the first color, and 8 edges in each of
the other two colors.
This result is already considered significant progress in the research
on Ramsey numbers as stated in~\cite{XuRad2015}.
To further apply this technique to determine if there exists a
$\tuple{13,8,8}$ regular $(4,3,3;30)$ Ramsey coloring requires to first
compute the currently unknown set $\RR(3,3,3;13)$.


Sections~\ref{sec:symBreak}---\ref{sec:33313b} present the second
step: computing $\RR(3,3,3;13)$. Section~\ref{sec:symBreak}
illustrates how a straightforward approach, combining SAT solving with
\emph{symmetry breaking}, works  for smaller instances but not for
$\RR(3,3,3;13)$.
Then Section~\ref{sec:abs} introduces an \emph{abstraction}, called
degree matrices, Section~\ref{sec:33313} demonstrates how to compute
degree matrices for $\RR(3,3,3;13)$, and Section~\ref{sec:33313b}
shows how to use the degree matrices to compute $\RR(3,3,3;13)$.
%
%
Section~\ref{sec:433_30} presents the third step re-examining the
embedding technique described in Section~\ref{sec:embed} which given the
set $\RR(3,3,3;13)$ applies to prove that there does not exist any
$(4,3,3;30)$ Ramsey coloring which is also $\tuple{13,8,8}$ regular.
%
%
Section~\ref{sec:conclude} presents a conclusion.


\section{Preliminaries and Notation}\label{sec:prelim}

In this paper, graphs are always simple, i.e.  undirected and with no
self loops. 
For a natural number $n$ let $[n]$ denote the set $\{1,2,\ldots,n\}$.
A graph coloring, in $k$ colors, is a pair $(G,\kappa)$ consisting of
a simple graph $G=(V,E)$ and a mapping $\kappa\colon E\to[k]$. When
$G$ is clear from the context we refer to $\kappa$ as the graph
coloring.
We typically represent $G=([n],E)$ as a (symmetric) $n\times n$ adjacency matrix,
$A$, defined such that
\[A_{i,j}=  
    \begin{cases}
       \kappa((i,j)) & \mbox{if } (i,j) \in E\\
          0          & \mbox{otherwise}
    \end{cases}
\]

Given a graph coloring $(G,\kappa)$ in $k$ colors with $G=(V,E)$,
the set of neighbors of a vertex $u\in V$ in color $c\in [k]$ is
$N_c(u) = \sset{v}{(u,v)\in E, \kappa((u,v))=c} $ and the color-$c$
degree of $u$ is $deg_{c}(u) = |N_c(u)|$. The color degree tuple of
$u$ is the $k$-tuple $deg(u)=\tuple{deg_{1}(u),\ldots,deg_{k}(u)}$.
The sub-graph of $G$ on the $c$ colored neighbors of $x\in V$ is the
projection of $G$ to vertices in $N_c(x)$ defined by $G^c_x =
(N_c(x),\sset{(u,v)\in E}{u,v\in N_c(x)})$.


For example, take as $G$ the graph coloring depicted by the adjacency
matrix in Figure~\ref{embed_12_8_8} with $u$ the vertex corresponding
to the first row in the matrix. Then,
$N_1(u) = \{2,3,4,5,6,7,8,9,10,11,12,13\}$, $N_2(u) =
\{14,15,16,17,18,19,20,21\}$, and
$N_3(u)=\{22,23,24,25,26,27,28,29\}$.
The subgraphs $G^1_u$, $G^2_u$, and $G^3_u$ are highlighted by the
boldface text in Figure~\ref{embed_12_8_8}.

An $(r_1,\ldots,r_k;n)$ Ramsey coloring is a graph coloring in $k$
colors of the complete graph $K_n$ that does not contain a
monochromatic complete sub-graph $K_{r_i}$ in color $i$ for each
$1\leq i\leq k$.
The set of all such colorings is denoted $\RR(r_1,\ldots,r_k;n)$.  The
Ramsey number $R(r_1,\ldots,r_k)$ is the least $n>0$ such that no
$(r_1,\ldots,r_k;n)$ coloring exists.
In the multicolor case ($k>2$), the only known value of a nontrivial
Ramsey number is $R(3,3,3)=17$.  
Prior to this paper, it was known that $30\leq R(4,3,3)\leq 31$.
Moreover, while the sets of $(3,3,3;n)$ colorings were known for
$14\leq n\leq 16$, the set of colorings for $n=13$ was never
published.\footnote{Recently, the set $\RR(3,3,3;13)$ has also been
  computed independently by: Stanislaw Radziszowski, Richard Kramer and
  Ivan Livinsky \cite{stas:personalcommunication}.}
More information on recent results concerning Ramsey numbers can be
found in the electronic dynamic survey by Radziszowski~\cite{Rad}.


\begin{figure}\small
  \centering
    \begin{eqnarray}
      \varphi_{adj}^{n,k}(A) &=& \hspace{-2mm}\bigwedge_{1\leq q<r\leq n}
                          \left(\begin{array}{l}
                             1\leq A_{q,r}\leq k  ~~\land~~
                             A_{q,r} = A_{r,q} ~~\land ~~
                             A_{q,q} = 0
                           \end{array}\right)
     \label{constraint:simple}
\\
     \varphi_{r}^{n,c}(A) &=& 
                \bigwedge_{I\in \wp_r([n])}
                \bigvee \sset{A_{i,j}\neq c}{i,j \in I, i<j}
     \label{constraint:nok}
\end{eqnarray}
\vspace{-3ex}
\begin{eqnarray}\small
\label{constraint:coloring}
    \varphi_{(r_1,\ldots,r_k;n)}(A) & = & \varphi_{adj}^{n,k}(A) \land  \hspace{-2mm}
                      \bigwedge_{1\leq c\leq k} \hspace{-1mm}
                      \varphi_{r_c}^{n,c}(A)
\end{eqnarray}

\caption{Constraints for graph labeling problems: Ramsey colorings
  $(r_1,\ldots,r_k;n)$}
  \label{fig:gcp}
\end{figure}


A graph coloring problem on $k$ colors is about the search for a graph
coloring which satisfies a given set of constraints. Formally, it is
specified as a
formula, $\varphi(A)$, where $A$ is an $n\times n$ adjacency matrix of
integer variables with domain $\{0\}\cup [k]$ and $\varphi$ is a
constraint on these variables.
A solution is an assignment of integer values to the variables in $A$
which satisfies $\varphi$ and determines both the graph edges and their
colors. We often refer to a solution as an integer adjacency matrix
and denote the set of solutions as $sol(\varphi(A))$.

Figure~\ref{fig:gcp} presents the $k$-color graph coloring problems we
focus on in this paper: $(r_1,\ldots,r_k;n)$ Ramsey colorings. 
Constraint~(\ref{constraint:simple}), $\varphi_{adj}^{n,k}(A)$, states
that the graph represented by matrix $A$ has $n$ vertices, is $k$
colored, and is simple.
Constraint~(\ref{constraint:nok}) $\varphi_{r}^{n,c}(A)$ states that
the $n\times n$ matrix $A$ has no embedded sub-graph $K_r$ in color
$c$. Each conjunct, one for each set $I$ of $r$ vertices, is a
disjunction stating that one of the edges between vertices of $I$ is
not colored $c$.  Notation: $\wp_r(S)$ denotes the set of all subsets
of size $r$ of the set $S$.
Constraint~(\ref{constraint:coloring}) states that $A$ is a
$(r_1,\ldots,r_k;n)$ Ramsey coloring.


For graph coloring problems, solutions are typically closed under
permutations of vertices and of colors. Restricting the search space
for a solution modulo such permutations is crucial when trying to
solve hard graph coloring problems. It is standard practice to
formalize this in terms of graph (coloring) isomorphism.

Let $G=(V,E)$ be a graph (coloring) with $V=[n]$ and let $\pi$ be a
permutation on $[n]$. Then $\pi(G) = (V,\sset{ (\pi(x),\pi(y))}{ (x,y)
  \in E})$. Permutations act on adjacency matrices in the natural way:
If $A$ is the adjacency matrix of a graph $G$, then $\pi(A)$ is the
adjacency matrix of $\pi(G)$ and $\pi(A)$ is obtained by
simultaneously permuting with $\pi$ both rows and columns of $A$.
 

\begin{definition}[\textbf{(weak) isomorphism of graph colorings}]
 \label{def:weak_iso}
 Let $(G,{\kappa_1})$ and $(H,{\kappa_2})$ be $k$-color graph
 colorings with $G=([n],E_1)$ and $H=([n],E_2)$. 
 We say that $(G,{\kappa_1})$ and $(H,{\kappa_2})$ are weakly
 isomorphic, denoted $(G,{\kappa_1})\approx(H,{\kappa_2})$ if there
 exist permutations $\pi \colon [n] \to [n]$ and $\sigma \colon [k]
 \to [k]$ such that $(u,v) \in E_1 \iff (\pi(u),\pi(v)) \in E_2$ and
 $\kappa_1((u,v)) = \sigma(\kappa_2((\pi(u), \pi(v))))$.
 We denote such a weak isomorphism:
 $(G,{\kappa_1})\approx_{\pi,\sigma}(H,{\kappa_2})$.
 When $\sigma$ is the identity permutation, we say that
 $(G,{\kappa_1})$ and $(H,{\kappa_2})$ are isomorphic.
\end{definition}
 
The following lemma emphasizes the importance of weak graph
isomorphism as it relates to Ramsey numbers. Many classic
coloring problems exhibit the same property.

\begin{lemma}[\textbf{$\RR(r_1,r_2,\ldots,r_k;n)$ is closed under
    $\approx$}]
\label{lemma:closed}
  \quad Let $(G,{\kappa_1})$ and $(H,{\kappa_2})$ be graph colorings
  in $k$ colors such that $(G,\kappa_1) \approx_{\pi,\sigma}
  (H,\kappa_2)$. Then,
  \[
  (G,\kappa_1) \in \RR(r_1,r_2,\ldots,r_k;n) \iff (H,\kappa_2) \in
  \RR(\sigma(r_1),\sigma(r_2),\ldots,\sigma(r_k);n)
  \]
\end{lemma}

We make use of the following theorem from~\cite{PR98}.


\begin{theorem}\label{thm:433}
  $30\leq R(4,3,3)\leq 31$ and, $R(4,3,3)=31$ if and only if there
  exists a $(4,3,3;30)$ coloring $\kappa$ of $K_{30}$ such that:
  (1) For every vertex $v$ and $i\in\{2,3\}$, $5\leq deg_{i}(v)\leq
  8$, and $13\leq deg_{1}(v)\leq 16$.
  (2) Every edge in the third color has at least one endpoint $v$ with
  $deg_{3}(v)=13$. 
  (3) There are at least 25 vertices $v$ for which 
  $deg_{1}(v)=13$, $deg_{2}(v)=deg_{3}(v)=8$.
\end{theorem}
\begin{corollary}\label{cor:degrees}
  Let $G=(V,E)$ be a $(4,3,3;30)$ coloring, $v\in V$ a selected
  vertex, and assume without loss of generality that $deg_2(v)\geq
  deg_3(v)$. Then, $deg(v)\in\set{\tuple{13, 8, 8},\tuple{14, 8,
      7},\tuple{15, 7, 7},\tuple{15, 8, 6},\tuple{16, 7, 6},\tuple{16,
      8, 5}}$.
\end{corollary}

Consider a vertex $v$ in a $(4,3,3;n)$ coloring and focus on the three
subgraphs induced by the neighbors of $v$ in each of the three
colors. The following states that these must be corresponding Ramsey
colorings.

\begin{observation}\label{obs:embed}
  Let $G$ be a $(4,3,3;n)$ coloring and $v$ be any vertex with
  $deg(v)=\tuple{d_1,d_2,d_3}$. Then,
  $d_1+d_2+d_3=n-1$ and $G^1_v$, $G^2_v$, and $G^3_v$ are respectively
  $(3,3,3;d_1)$, $(4,2,3;d_2)$, and $(4,3,2;d_3)$ colorings.
\end{observation}

Note that by definition a $(4,2,3;n)$ coloring is a $(4,3;n)$ Ramsey
coloring in colors 1 and 3 and likewise a $(4,3,2;n)$ Ramsey coloring
is a $(4,3;n)$ coloring in colors 1 and 2. This is because the ``2''
specifies that the coloring does not contain a subgraph $K_2$ in the
corresponding color and this means that it contains no edge with that
color.
For $n\in\{14,15,16\}$, the sets $\RR(3,3,3;n)$ are known and consist
respectively of 115, 2, and 2 colorings.
Similarly, for $n\in\{5,6,7,8\}$ the sets $\RR(4,3;n)$ are known and
consist respectively of 9, 15, 9, and 3 colorings.

In this paper computations are performed using the
CryptoMiniSAT~\cite{Crypto} SAT solver. SAT encodings (CNF) are
obtained using the finite-domain constraint compiler
\bee~\cite{jair2013}. The use of \bee\ facilitates applications to
find a single (first) solution, or to find all solutions for a
constraint, modulo a specified set of variables.
When solving for all solutions, our implementation iterates with the
SAT solver, adding so called \emph{blocking clauses} each time another
solution is found. This technique, originally due to
McMillan~\cite{McMillan2002}, is simplistic but suffices for our
purposes.
All computations were performed on a cluster with a total of $228$
Intel E8400 cores clocked at 2 GHz each, able to run a total of $456$
parallel threads. Each of the cores in the cluster has computational
power comparable to a core on a standard desktop computer.  Each SAT
instance is run on a single thread.

\section{Basic SAT Encoding and Embeddings}
\label{sec:embed}


Throughout the paper we apply a SAT solver to solve CNF encodings of
constraints such as those presented in Figure~\ref{fig:gcp}. In this
way it is straightforward to find a Ramsey coloring or prove its
non-existence. Ours is a standard encoding to CNF. To this end:
nothing new.
For an $n$ vertex graph coloring problem in $k$ colors we take an
$n\times n$ matrix $A$ where $A_{i,j}$ represents in $k$ bits the
edge $(i,j)$ in the graph: exactly one bit is true indicating which
color the edge takes, or no bit is true indicating that the edge
$(i,j)$ is not in the graph.  
Already at the representation level, we use the same Boolean variables
to represent the color in $A_{i,j}$ and in $A_{j,i}$ for each $1\leq
i<j\leq n$. We further fix the variables corresponding to $A_{i,i}$ to
$\false$.  The rest of the SAT encoding is straightforward.

Constraint~(\ref{constraint:simple}) is encoded to CNF by introducing
clauses to state that for each $A_{i,j}$ with $1\leq i<j\leq n$ at
most one of the $k$ bits representing the color of the edge $(i,j)$ is
true. In our setting typically $k=3$. For three colors, if
$b_1,b_2,b_3$ are the  bits representing the color of an edge,
then three clauses suffice: $(\bar b_1\lor \bar b_2),(\bar b_1\lor
\bar b_3),(\bar b_2\lor \bar b_3)$.
Constraint~(\ref{constraint:nok}) is encoded by a single clause 
per set $I$ of $r$ vertices expressing that at least one
of the bits corresponding to an edge between vertices in $I$ does not
have color $c$. 
Finally Constraint~(\ref{constraint:coloring}) is a conjunction of
constraints of the previous two forms. 


In Section~\ref{sec:symBreak} we will improve on this basic encoding by
introducing symmetry breaking constraints (encoded to CNF). However,
for now we note that, even with symmetry breaking constraints, using
the basic encoding, a SAT solver is currently not able to solve any of
the open Ramsey coloring problems such as those considered in this
paper. In particular, directly applying a SAT solver to search for a
$(4,3,3;30)$ Ramsey coloring is hopeless.

To facilitate the search for $(4,3,3;30)$ Ramsey coloring using a SAT
encoding, we apply a general approach where, when seeking a
$(r_1,\ldots,r_k;n)$ Ramsey coloring one selects a ``preferred''
vertex, call it $v_1$, and based on its degrees in each of the $k$
colors, embeds $k$ subgraphs which are corresponding smaller
colorings. Using this approach, we apply Corollary~\ref{cor:degrees}
and Observation~\ref{obs:embed} to establish that a $(4,3,3;30)$
coloring, if one exists, must be $\tuple{13,8,8}$
regular. Specifically, all vertices must have 13 edges in the first
color and 8 each, in the second and third colors. This result is
considered significant progress in the research on Ramsey
numbers~\cite{XuRad2015}. This ``embedding'' approach is often applied
in the Ramsey number literature where the process of completing (or
trying to complete) a partial solution (an embedding) to a Ramsey
coloring is called \emph{gluing}.
See for example the presentations in~\cite{PiwRad2001,FKRad2004,PR98}.

\begin{theorem}\label{thm:regular}
  Any $(4,3,3;30)$ coloring, if one exists, is $\tuple{13,8,8}$ regular.
\end{theorem}

\begin{proof}
  By computation as described in the rest of this section.
\end{proof}

\newcommand{\undA}{\_&\_&\_&\_&\_&\_&\_&\_&\_&\_&\_&\_&\_&\_&\_}
\newcommand{\undB}{\_&\_&\_&\_&\_&\_&\_&\_&\_&\_&\_&\_&\_&\_}
\newcommand{\undC}{\_&\_&\_&\_&\_&\_&\_}
\newcommand{\undD}{\_&\_&\_&\_&\_&\_&\_&\_&\_&\_&\_&\_&\_&\_&\_&\_&\_&\_&\_&\_&\_&\_}
\begin{wrapfigure}[14]{r}{.460\linewidth}\vspace{-11mm}
  \begin{center}                   
\resizebox{.95\linewidth}{!}{$
\left[                                  
\begin{smallmatrix}
\bz&\ba&\ba&\ba&\ba&\ba&\ba&\ba&\ba&\ba&\ba&\ba&\ba&\ba&\ba&
    \bb&\bb&\bb&\bb&\bb&\bb&\bb&\bb&\bc&\bc&\bc&\bc&\bc&\bc&\bc&\\
\ba& 0&1&1&1&1&2&2&2&2&3&3&3&3&3&\undA \\
\ba& 1&0&2&2&3&1&1&3&3&1&2&2&3&3&\undA \\
\ba& 1&2&0&3&2&1&2&1&3&1&1&3&2&3&\undA \\
\ba& 1&2&3&0&2&2&1&3&1&1&3&1&3&2&\undA \\
\ba& 1&3&2&2&0&1&1&2&2&2&3&3&1&1&\undA \\
\ba& 2&1&1&2&1&0&3&3&1&2&3&1&2&3&\undA \\
\ba& 2&1&2&1&1&3&0&1&3&2&1&3&3&2&\undA \\
\ba& 2&3&1&3&2&3&1&0&1&3&2&3&2&1&\undA \\
\ba& 2&3&3&1&2&1&3&1&0&3&3&2&1&2&\undA \\
\ba& 3&1&1&1&2&2&2&3&3&0&2&2&1&1&\undA \\
\ba& 3&2&1&3&3&3&1&2&3&2&0&1&1&2&\undA \\
\ba& 3&2&3&1&3&1&3&3&2&2&1&0&2&1&\undA \\
\ba& 3&3&2&3&1&2&3&2&1&1&1&2&0&2&\undA \\
\ba& 3&3&3&2&1&3&2&1&2&1&2&1&2&0&\undA \\
\bb& \undB&0&3&3&3&1&1&1&1&\undC \\
\bb& \undB&3&0&1&1&3&3&1&1&\undC \\
\bb& \undB&3&1&0&1&3&1&3&1&\undC \\
\bb& \undB&3&1&1&0&1&1&3&3&\undC \\
\bb& \undB&1&3&3&1&0&1&1&3&\undC \\
\bb& \undB&1&3&1&1&1&0&3&3&\undC \\
\bb& \undB&1&1&3&3&1&3&0&1&\undC \\
\bb& \undB&1&1&1&3&3&3&1&0&\undC \\
\bc&    \undD& 0&2&2&1&1&1&1 \\
\bc&    \undD& 2&0&1&2&1&1&1 \\
\bc&    \undD& 2&1&0&1&2&1&1 \\
\bc&    \undD& 1&2&1&0&1&2&1 \\
\bc&    \undD& 1&1&2&1&0&1&2 \\
\bc&    \undD& 1&1&1&2&1&0&2 \\
\bc&    \undD& 1&1&1&1&2&2&0 \\                   
\end{smallmatrix}\right]$}
\caption{One  embedding to search for a $(4,3,3;30)$
  coloring when $deg(v_1)=\tuple{14,8,7}$.   }
        \label{embed_14_8_7}
  \end{center} 
\end{wrapfigure}


We seek a $(4,3,3;30)$ coloring of $K_{30}$, represented as a
$30\times 30$ adjacency matrix $A$.  Let $v_1$ correspond to the the
first row in $A$ with $deg(v_1)=\tuple{d_1,d_2,d_3}$ as prescribed by
Corollary~\ref{cor:degrees}.  For each possible triplet
$\tuple{d_1,d_2,d_3}$, except $\tuple{13,8,8}$, we take each of the
known corresponding colorings for the subgraphs $G^1_{v_1}$,
$G^2_{v_1}$, and $G^3_{v_1}$ and embed them into $A$. We then apply a
SAT solver, to (try to) complete the remaining cells in $A$ to satisfy
$\varphi_{4,3,3;30}(A)$ as defined by
Constraint~(\ref{constraint:coloring}) of Figure~\ref{fig:gcp}. If the
SAT solver fails, then no such completion exists.

To illustrate the approach, consider the case where
$deg(v_1)=\tuple{14,8,7}$. Figure~\ref{embed_14_8_7} details one of
the embeddings corresponding to this case.  The first row and column
of $A$ specify the colors of the edges of the 29 neighbors of $v_1$
(in bold). The symbol ``$\_$'' indicates an integer variable that
takes a value between 1 and 3.
The neighbors of $v_1$ in color~1 form a submatrix of $A$ embedded in
rows (and columns) 2--15 of the matrix in the Figure.  By
Corollary~\ref{obs:embed} these are a $(3,3,3;14)$ Ramsey coloring and
there are 115 possible such colorings modulo weak isomorphism. The
Figure details one of them.
Similarly, there are 3 possible $(4,2,3;8)$ colorings which are
subgraphs for the neighbors of $v_1$ in color~2. In
Figure~\ref{embed_14_8_7}, rows (and columns) 16--23 detail one such
coloring.
Finally, there are 9 possible $(4,3,2;7)$ colorings which are
subgraphs for the neighbors of $v_1$ in color~3. In
Figure~\ref{embed_14_8_7}, rows (and columns) 24--30 detail one such
coloring.

To summarize, Figure~\ref{embed_14_8_7} is a partially instantiated
adjacency matrix. The first row determines the degrees of
$v_1$, in the three colors, and 3 corresponding subgraphs
are embedded.  The uninstantiated values in the matrix must be
completed to obtain a solution that satisfies $\varphi_{4,3,3;30}(A)$
as specified in
Constraint~(\ref{constraint:coloring}) of Figure~\ref{fig:gcp}. This can be
determined using a SAT solver.
%
%
For the specific example in Figure~\ref{embed_14_8_7}, the CNF
generated using our tool set consists of 33{,}959 clauses, involves
5{,}318 Boolean variables, and is shown to be unsatisfiable in 52 seconds of
computation time.
For the case where $v_1$ has degrees $\tuple{14,8,7}$ in the three
colors this is one of $115\times 3\times 9 = 3105$ instances that need
to be checked.

Table~\ref{table:regular} summarizes the experiment which proves
Theorem~\ref{thm:regular}. For each of the possible degrees of
vertex~1 in a $(4,3,3;30)$ coloring as prescribed by
Corollary~\ref{cor:degrees}, except $\tuple{13,8,8}$, and for each
possible choice of colorings for the derived subgraphs $G^1_{v_1}$,
$G^2_{v_1}$, and $G^3_{v_1}$, we apply a SAT solver to show that the
instance $\varphi_{(4,3,3;30)}(A)$ of
Constraint~(\ref{constraint:coloring}) of Figure~\ref{fig:gcp} cannot
be satisfied. The table details for each degree triple, the number of
instances, their average size (number of clauses and Boolean
variables), and the average and total times to show that the
constraint is not satisfiable.

\begin{table}[t]
\begin{center}\scriptsize
\begin{tabular}{|c|rl|c|c|r|r|}
\hline
$v_1$ degrees    & \multicolumn{2}{c|}{\# instances} & \# clauses (avg.) & \# vars (avg.) &
unsat (avg) & unsat (total) \\ \hline
\hline
(16,8,5) & 54 &=~2*3*9   & 32432 & 5279 &  51 sec. &  0.77 hrs.\\ 
\hline
(16,7,6) & 270 &=~2*9*15 & 32460 & 5233 & 420 sec. & 31.50 hrs.\\ 
\hline
(15,8,6) & 90 &=~2*3*15  & 33607 & 5450 &  93 sec. &  2.32 hrs.\\ 
\hline
(15,7,7) & 162 &=~2*9*9  & 33340 & 5326 &1554 sec. & 69.94 hrs.\\ 
\hline
(14,8,7) &3105 &=~115*3*9& 34069 & 5324 & 294 sec. &253.40 hrs.\\ 
\hline
\end{tabular}
\end{center}
\caption{Proving that any $(4,3,3;30)$ Ramsey coloring is $\tuple{13,8,8}$
  regular (summary).}
\label{table:regular}
\end{table}

All of the SAT instances described in the experiment summarized by
Table~\ref{table:regular} are unsatisfiable. The solver reports
``unsat''.  To gain confidence in our implementation, we illustrate
its application on a satisfiable instance: to find a, known to exist,
$(4,3,3;29)$ coloring.  This experiment involves some reverse
engineering.
%
\begin{wrapfigure}[16]{r}{.460\linewidth}\vspace{-10mm}
\newcommand{\uuA}{\_&\_&\_&\_&\_&\_&\_&\_&\_&\_&\_&\_&\_&\_&\_&\_}
\newcommand{\uuB}{\_&\_&\_&\_&\_&\_&\_&\_&\_&\_&\_&\_}
\newcommand{\uuC}{\_&\_&\_&\_&\_&\_&\_&\_}
\newcommand{\uuD}{\_&\_&\_&\_&\_&\_&\_&\_&\_&\_&\_&\_&\_&\_&\_&\_&\_&\_&\_&\_}

\newcommand{\ga}{\textcolor{gray}{1}} 
\newcommand{\gb}{\textcolor{gray}{2}} 
\newcommand{\gc}{\textcolor{gray}{3}} 

\begin{center}
  \resizebox{.95\linewidth}{!}{$\left[
      \begin{smallmatrix}
        \bz&\ba&\ba&\ba&\ba&\ba&\ba&\ba&\ba&\ba&\ba&\ba&\ba&\bb&\bb&\bb&\bb&\bb&\bb&\bb&\bb&\bc&\bc&\bc&\bc&\bc&\bc&\bc&\bc\\
        \ba&\bz&\ba&\bc&\ba&\bc&\bb&\bb&\bb&\bb&\ba&\bc&\ba&\gb&\gb&\ga&\gc&\ga&\ga&\ga&\gb&\gc&\gc&\ga&\gb&\gc&\ga&\ga&\gc\\
        \ba&\ba&\bz&\ba&\bc&\bc&\ba&\bb&\ba&\bc&\bb&\ba&\bc&\ga&\gb&\gc&\gc&\ga&\gb&\gb&\gb&\ga&\gc&\ga&\ga&\gc&\ga&\gb&\gb\\
        \ba&\bc&\ba&\bz&\ba&\bb&\bc&\bb&\bb&\bc&\bc&\bb&\ba&\ga&\ga&\gc&\ga&\ga&\gb&\gc&\gb&\ga&\ga&\gc&\gc&\gb&\ga&\gb&\ga\\
        \ba&\ba&\bc&\ba&\bz&\ba&\bc&\ba&\bb&\ba&\bc&\bc&\bb&\gc&\gb&\gb&\gb&\ga&\gb&\gc&\ga&\gb&\ga&\gc&\ga&\ga&\ga&\gb&\gc\\
        \ba&\bc&\bc&\bb&\ba&\bz&\ba&\bc&\bc&\bb&\bb&\ba&\bb&\ga&\gb&\ga&\gb&\ga&\gc&\gc&\ga&\gb&\ga&\gb&\gc&\ga&\gc&\ga&\ga\\
        \ba&\bb&\ba&\bc&\bc&\ba&\bz&\ba&\bc&\ba&\bb&\bb&\bb&\ga&\ga&\ga&\gb&\gb&\gb&\ga&\gc&\gc&\gc&\gb&\ga&\ga&\gc&\gc&\ga\\
        \ba&\bb&\bb&\bb&\ba&\bc&\ba&\bz&\ba&\bc&\bc&\ba&\bb&\gc&\ga&\gb&\gb&\gb&\ga&\ga&\ga&\ga&\gc&\gc&\ga&\ga&\gb&\gc&\gc\\
        \ba&\bb&\ba&\bb&\bb&\bc&\bc&\ba&\bz&\ba&\bb&\bc&\bc&\ga&\gc&\gc&\ga&\gb&\ga&\gb&\ga&\ga&\ga&\gc&\ga&\gc&\gb&\ga&\gb\\
        \ba&\bb&\bc&\bc&\ba&\bb&\ba&\bc&\ba&\bz&\ba&\bc&\ba&\ga&\gc&\gb&\ga&\gb&\gb&\gb&\gc&\gc&\gb&\ga&\ga&\ga&\gc&\ga&\gb\\
        \ba&\ba&\bb&\bc&\bc&\bb&\bb&\bc&\bb&\ba&\bz&\ba&\bc&\gb&\gc&\gb&\ga&\ga&\gc&\ga&\ga&\ga&\gb&\ga&\gb&\gc&\gc&\ga&\ga\\
        \ba&\bc&\ba&\bb&\bc&\ba&\bb&\ba&\bc&\bc&\ba&\bz&\ba&\gb&\gb&\gb&\ga&\gc&\gc&\gb&\ga&\gb&\gb&\ga&\gc&\ga&\ga&\gc&\ga\\
        \ba&\ba&\bc&\ba&\bb&\bb&\bb&\bb&\bc&\ba&\bc&\ba&\bz&\gb&\ga&\ga&\ga&\gc&\ga&\gb&\gb&\gc&\ga&\ga&\gc&\gb&\ga&\gc&\gc\\
        \bb&\gb&\ga&\ga&\gc&\ga&\ga&\gc&\ga&\ga&\gb&\gb&\gb&\bz&\ba&\bc&\bc&\bc&\ba&\ba&\ba&\gc&\ga&\ga&\gc&\gc&\gb&\gb&\gb\\
        \bb&\gb&\gb&\ga&\gb&\gb&\ga&\ga&\gc&\gc&\gc&\gb&\ga&\ba&\bz&\ba&\bc&\ba&\ba&\bc&\ba&\ga&\gc&\ga&\gc&\gb&\gb&\ga&\gc\\
        \bb&\ga&\gc&\gc&\gb&\ga&\ga&\gb&\gc&\gb&\gb&\gb&\ga&\bc&\ba&\bz&\ba&\ba&\bc&\ba&\ba&\ga&\ga&\gb&\ga&\gb&\gc&\gc&\gc\\
        \bb&\gc&\gc&\ga&\gb&\gb&\gb&\gb&\ga&\ga&\ga&\ga&\ga&\bc&\bc&\ba&\bz&\ba&\ba&\ba&\bc&\ga&\gb&\gc&\gb&\ga&\gc&\gc&\gb\\
        \bb&\ga&\ga&\ga&\ga&\ga&\gb&\gb&\gb&\gb&\ga&\gc&\gc&\bc&\ba&\ba&\ba&\bz&\ba&\bc&\bc&\gb&\gc&\gc&\ga&\gb&\gc&\gb&\ga\\
        \bb&\ga&\gb&\gb&\gb&\gc&\gb&\ga&\ga&\gb&\gc&\gc&\ga&\ba&\ba&\bc&\ba&\ba&\bz&\ba&\bc&\gc&\gc&\gc&\gb&\ga&\gb&\ga&\ga\\
        \bb&\ga&\gb&\gc&\gc&\gc&\ga&\ga&\gb&\gb&\ga&\gb&\gb&\ba&\bc&\ba&\ba&\bc&\ba&\bz&\ba&\gc&\ga&\gb&\gb&\gc&\ga&\gc&\ga\\
        \bb&\gb&\gb&\gb&\ga&\ga&\gc&\ga&\ga&\gc&\ga&\ga&\gb&\ba&\ba&\ba&\bc&\bc&\bc&\ba&\bz&\gb&\gb&\gb&\gc&\gc&\ga&\ga&\gc\\
        \bc&\gc&\ga&\ga&\gb&\gb&\gc&\ga&\ga&\gc&\ga&\gb&\gc&\gc&\ga&\ga&\ga&\gb&\gc&\gc&\gb&\bz&\ba&\ba&\bb&\ba&\bb&\ba&\bb\\
        \bc&\gc&\gc&\ga&\ga&\ga&\gc&\gc&\ga&\gb&\gb&\gb&\ga&\ga&\gc&\ga&\gb&\gc&\gc&\ga&\gb&\ba&\bz&\ba&\ba&\bb&\bb&\bb&\ba\\
        \bc&\ga&\ga&\gc&\gc&\gb&\gb&\gc&\gc&\ga&\ga&\ga&\ga&\ga&\ga&\gb&\gc&\gc&\gc&\gb&\gb&\ba&\ba&\bz&\ba&\ba&\bb&\bb&\bb\\
        \bc&\gb&\ga&\gc&\ga&\gc&\ga&\ga&\ga&\ga&\gb&\gc&\gc&\gc&\gc&\ga&\gb&\ga&\gb&\gb&\gc&\bb&\ba&\ba&\bz&\bb&\ba&\bb&\ba\\
        \bc&\gc&\gc&\gb&\ga&\ga&\ga&\ga&\gc&\ga&\gc&\ga&\gb&\gc&\gb&\gb&\ga&\gb&\ga&\gc&\gc&\ba&\bb&\ba&\bb&\bz&\ba&\ba&\bb\\
        \bc&\ga&\ga&\ga&\ga&\gc&\gc&\gb&\gb&\gc&\gc&\ga&\ga&\gb&\gb&\gc&\gc&\gc&\gb&\ga&\ga&\bb&\bb&\bb&\ba&\ba&\bz&\ba&\ba\\
        \bc&\ga&\gb&\gb&\gb&\ga&\gc&\gc&\ga&\ga&\ga&\gc&\gc&\gb&\ga&\gc&\gc&\gb&\ga&\gc&\ga&\ba&\bb&\bb&\bb&\ba&\ba&\bz&\ba\\
        \bc&\gc&\gb&\ga&\gc&\ga&\ga&\gc&\gb&\gb&\ga&\ga&\gc&\gb&\gc&\gc&\gb&\ga&\ga&\ga&\gc&\bb&\ba&\bb&\ba&\bb&\ba&\ba&\bz
      \end{smallmatrix}\right]$}
\end{center}
\caption{Embedding (boldface) and solution (gray text) for a
  $(4,3,3;29)$ Ramsey coloring.}
        \label{embed_12_8_8}
\end{wrapfigure}
In 1966 Kalbfleisch~\cite{kalb66} reported the existence of a
circulant $(3,4,4;29)$ coloring. Encoding instance
$\varphi_{(4,3,3;29)}(A)$ of Constraint~(\ref{constraint:coloring})
together with a constraint that states that the adjacency
matrix $A$ is circulant, results in a CNF with 146{,}506 clauses and
8{,}394 variables. Using a SAT solver, we obtain a corresponding
$(4,3,3;29)$ coloring in less than two seconds of computation
time. The solution is $\tuple{12,8,8}$ regular and isomorphic to the
adjacency matrix depicted as Figure~\ref{embed_12_8_8}. 
Now we apply the embedding approach.  We take the partial solution
(the boldface elements) corresponding to the three subgraphs:
$G^1_{v_1}$, $G^2_{v_1}$ and $G^3_{v_1}$ which are respectively
$(3,3,3;12)$, $(4,2,3;8)$ and $(4,3,2;8)$ Ramsey colorings.  Applying
a SAT solver to complete this partial solution to a $(4,3,3;29)$
coloring satisfying Constraint~(\ref{constraint:coloring}) involves a
CNF with 30{,}944 clauses and 4{,}736 variables and requires under two
hours of computation time. Figure~\ref{embed_12_8_8} portrays the
solution (the gray elements).

To apply the embedding approach described in this section to determine
if there exists a $(4,3,3;30)$ Ramsey coloring which is
$\tuple{13,8,8}$ regular would require access to the set
$\RR(3,3,3;13)$.  We defer this discussion until after
Section~\ref{sec:33313b} where we describe how we compute the set of
all 78{,}892 $(3,3,3;13)$ Ramsey colorings modulo weak isomorphism.

\section{Symmetry Breaking: Computing $\RR(r_1,\ldots,r_k;n)$}
\label{sec:symBreak}


In this section we prepare the ground to apply a SAT solver to find
the set of all $(r_1,\ldots,r_k;n)$ Ramsey colorings modulo weak
isomorphism. 
The constraints are those presented in Figure~\ref{fig:gcp} and
their encoding to CNF is as described in Section~\ref{sec:embed}. Our
final aim is to compute the set of all $(3,3,3;13)$ colorings modulo
weak isomorphism. Then we can apply the embedding technique of
Section~\ref{sec:embed} to determine the existence of a
$\tuple{13,8,8}$ regular $(4,3,3;30)$ Ramsey coloring. Given
Theorem~\ref{thm:regular}, this will determine the value of
$R(4,3,3)$.

Solving hard search problems on graphs, and graph coloring problems in
particular, relies heavily on breaking symmetries in the search space.
When searching for a graph, the names of the vertices do not
matter, and restricting the search modulo graph isomorphism is highly
beneficial. When searching for a graph coloring, on top of graph
isomorphism, solutions are typically closed under permutations of the
colors: the names of the colors do not matter and the term
often used is ``weak isomorphism''~\cite{PR98} (the equivalence
relation is weaker because both node names and edge colors do not
matter).
When the problem is to compute the set of all solutions modulo (weak)
isomorphism the task is even more challenging. Often one first
attempts to compute all the solutions of the coloring problem, and to
then apply one of the available graph isomorphism tools, such as
\texttt{nauty}~\cite{nauty} to select representatives of their
equivalence classes modulo (weak) isomorphism. This is a
\emph{generate and test} approach. However, typically the number of
solutions is so large that this approach is doomed to fail even though
the number of equivalence classes itself is much smaller. The problem
is that tools such as \texttt{nauty} apply after, and not during,
generation.
To this end, we follow~\cite{CodishMPS14} where Codish \etal\ show
that the symmetry breaking approach of
\cite{DBLP:conf/ijcai/CodishMPS13} holds also for graph coloring
problems where the adjacency matrix consists of integer variables.
This is a \emph{constrain and generate approach}. But, as symmetry
breaking does not break all symmetries, it is still necessary to
perform some reduction using a tool like \texttt{nauty}.\footnote{Note
  that \texttt{nauty} does not directly handle edge colored graphs and
  weak isomorphism directly. We applied an approach called
  $k$-layering described by Derrick Stolee~\cite{Stolee}.}  This form
of symmetry breaking is an important component in our methodology.

\begin{definition}\textbf{\cite{DBLP:conf/ijcai/CodishMPS13}.}
\label{def:SBlexStar}
Let $A$ be an $n\times n$ adjacency matrix. Then, 
\begin{equation}\label{eq:symbreak}
 \SB^*_\ell(A) = \bigwedge\sset 
  {A_{i}\preceq_{\{i,j\}}A_{j}}{i<j}
\end{equation}
where $A_{i}\preceq_{\{i,j\}}A_{j}$ denotes the lexicographic order
between the $i^{th}$ and $j^{th}$ rows of $A$ (viewed as strings)
omitting the elements at positions $i$ and~$j$ (in both rows).
\end{definition} 

We omit the precise details of how Constraint~(\ref{eq:symbreak}) is
encoded to CNF. In our implementation this is performed by the finite
domain constraint compiler \bee\ and details can be found in
\cite{jair2013}.  Table~\ref{tab:333n1} illustrates the impact of the
symmetry breaking Constraint~(\ref{eq:symbreak}) on the search for the
Ramsey colorings required in the proof of
Theorem~\ref{thm:regular}.

\begin{table}[t]
  \centering{\scriptsize
  \begin{tabular}{|c|r|rrrr|rrrr|}
\hline
Instance & \#${\setminus}_{\approx}$ & \multicolumn{4}{c|}{no sym break} 
                               & \multicolumn{4}{c|}{with sym break}
\\
\hline
  &&      vars & clauses & time\quad & \#\quad      
          & vars & clauses & time\quad & \#\quad\\
\hline
(4,3;5)& 9 &10&15&0.02&322 &24&85&0.01&13\\
(4,3;6)&15 &15&35&0.35&2812 &48&200&0.01&31\\
(4,3;7)& 9 &21&70&9.27&13842 &85&390&0.01&45\\
(4,3;8)& 3 &28&126&19.46&17640 &138&676&0.01&20\\
\hline
(3,3,3;16) & 2   &360 & 2160 &\quad$\infty~$&\quad?~~~ &3328  & 17000  &0.14      &6\\
(3,3,3;15) & 2   &315 & 1785 &\quad$\infty~$&\quad?~~~ &2707  & 13745  &0.37      &66\\
(3,3,3;14) & 115 &273 & 1456 &\quad$\infty~$&\quad?~~~ &2169  & 10936
&~~259.56  &~~24635\\
\hline
(3,3,3;13) & ?~  &234 & 1170 &\quad$\infty~$&\quad?~~~  &1708 & 8540   &$\infty\quad$ &\quad?~~~\\
    \hline
  \end{tabular}}
\caption{Computing  Ramsey colorings with
  and without the symmetry break Constraint~(\ref{eq:symbreak})
  (time in seconds with 24 hr. timeout marked by $\infty$).}
\label{tab:333n1}
\end{table}

The first four rows in the table portray the required instances of the forms
$(4,3,2;n)$ and $(4,2,3;n)$ which by definition correspond to
$(4,3;n)$ colorings (respectively in colors 1 and 3, and in colors 1
and 2).
The next three rows correspond to $(3,3,3;n)$ colorings where
$n\in\{14,15,16\}$.
The last row illustrates our
failed attempt to apply a SAT encoding to compute $\RR(3,3,3;13)$.
%
The first column in the table specifies the instance.  The column
headed by ``\#${\setminus}_{\approx}$'' specifies the known (except
for the last row) number of colorings modulo weak
isomorphism~\cite{Rad}.
The columns headed by ``vars'' and ``clauses'' indicate, the
numbers of variables and clauses in the corresponding CNF encodings of
the coloring problems with and without the symmetry breaking
Constraint~(\ref{eq:symbreak}).  The columns headed by ``time''
indicate the time (in seconds) to find all colorings iterating with a
SAT solver. The timeout assumed here is 24 hours. The column headed by
``\#'' specifies the number of colorings found by iterated SAT solving.

In the first four rows, notice the impact of symmetry breaking which
reduces the number of solutions by 1--3 orders of magnitude. In the
next three rows the reduction is more acute. Without symmetry breaking
the colorings cannot be computed within the 24 hour timeout. The sets
of colorings obtained with symmetry breaking have been verified to
reduce, using \texttt{nauty}~\cite{nauty}, to the known number of
colorings modulo weak isomorphism indicated in the second column.

\section{Abstraction: Degree Matrices for Graph Colorings}
\label{sec:abs}

This section introduces an abstraction on graph colorings defined in
terms of \emph{degree matrices}. The motivation is to solve a hard
graph coloring problem by first searching for its degree matrices.
Degree matrices are to graph coloring problems as degree
sequences~\cite{ErdosGallai1960} are to graph search problems. A
degree sequence is a monotonic nonincreasing sequence of the vertex
degrees of a graph. A graphic sequence is a sequence which can be the
degree sequence of some graph.

The idea underlying our approach is that when the combinatorial
problem at hand is too hard, then possibly solving an abstraction of
the problem is easier. In this case, a solution of the abstract
problem can be used to facilitate the search for a solution of the
original problem.


\begin{definition}[\textbf{degree matrix}]
\label{def:dm}
  Let $A$ be a graph coloring on $n$ vertices with $k$ colors. The
  \emph{degree matrix} of $A$, denoted $dm(A)$ is an $n\times k$
  matrix, $M$ such that $M_{i,j} = deg_j(i)$ is the degree of vertex
  $i$ in color $j$. 
\end{definition}

\begin{wrapfigure}[7]{r}{.33\linewidth}\vspace{-11mm}
\begin{center}
  \resizebox{.80\linewidth}{!}{$\left[\left.
      \begin{smallmatrix}
        12 & 8 & 8 \\
        ~ \vdots\\[1ex]
        12 & 8 & 8 \\
      \end{smallmatrix}\right]\right\}\mbox{\scriptsize29 rows}$}
  \end{center}
\caption{\small A  degree matrix.}
\label{fig:dm}
\end{wrapfigure}
Figure~\ref{fig:dm} illustrates the degree matrix of the graph
coloring given as Figure~\ref{embed_12_8_8}. The three columns
correspond to the three colors and the 29 rows to the 29 vertices. The
degree matrix consists of 29 identical rows as the corresponding graph
coloring is $\tuple{12,8,8}$ regular.

A degree matrix $M$ represents the set of graphs $A$ such that
$dm(A)=M$. 
Due to properties of weak-isomorphism (vertices as well as colors can
be reordered) we can exchange both rows and columns of a degree matrix
without changing the set of graphs it represents.  In the rest of our
construction we adopt a representation in which the rows and columns
of a degree matrix are sorted lexicographically.

\begin{definition}[\textbf{lex sorted degree matrix}]
  For an $n\times k$ degree matrix $M$ we denote by $lex(M)$ the smallest
  matrix with rows and columns in the lexicographic order
  (non-increasing) obtained by permuting rows and columns of $M$.
\end{definition}

\begin{definition}[\textbf{abstraction}]
\label{def:abs}
Let $A$ be a graph coloring on $n$ vertices with $k$ colors. The
\emph{abstraction} of $A$ to a degree matrix is
$\alpha(A)=lex(dm(A))$. For a set $\AA$ of graph colorings we denote
$\alpha(\AA) = \sset{\alpha(A)}{A\in\AA}$.
\end{definition}

Note that if $A$ and $A'$ are weakly isomorphic, then
$\alpha(A)=\alpha(A')$.

\begin{definition}[\textbf{concretization}]
\label{def:conc}
  Let $M$ be an $n\times k$ degree matrix. Then, $\gamma(M) =
  \sset{A}{\alpha(A)=M}$ is the set of graph colorings
  represented by $M$. For a set $\MM$ of degree matrices we denote
  $\gamma(\MM) = \cup\sset{\gamma(M)}{M\in\MM}$.
\end{definition}

Let $\varphi(A)$ be a graph coloring problem in $k$ colors on an
$n\times n$ adjacency matrix, $A$.
Our strategy to compute $\AA=sol(\varphi(A))$ is to first compute an
over-approximation $\MM$ of degree matrices such that
$\gamma(\MM)\supseteq\AA$ and to then use $\MM$ to guide the
computation of $\AA$.
We denote the set of solutions of the graph coloring problem,
$\varphi(A)$, which have a given degree matrix, $M$, by
$sol_M(\varphi(A))$. Then
\begin{eqnarray}
\label{eq:approx}   
       sol(\varphi(A)) &=&  \bigcup_{M\in\MM} sol_M(\varphi(A))\\
\label{eq:solM}
	sol_M(\varphi(A)) & = & sol(\varphi(A)\wedge\alpha(A){=}M)
\end{eqnarray}
%
%

Equation~(\ref{eq:approx}) implies that, we can compute the solutions to a
graph coloring problem $\varphi(A)$ by computing the independent sets
$sol_M(\varphi(A))$ for any over approximation $\MM$ 
of the degree matrices of the solutions of $\varphi(A)$.
%
This facilitates the computation for two reasons:
(1) The problem is now broken into a set of independent sub-problems
for each $M\in\MM$ which can be solved in parallel, and
(2) The computation of each individual $sol_M(\varphi(A))$ is now
directed using $M$. 

The constraint $\alpha(A){=}M$ in the right side of
Equation~(\ref{eq:solM}) is encoded to SAT by introducing (encodings
of) cardinality constraints. For each row of the matrix $A$ the
corresponding row in $M$ specifies the number of elements with value
$c$ (for $1\leq c\leq k$) that must be in that row. We omit the
precise details of the encoding to CNF. In our implementation this is
performed by the finite domain constraint compiler \bee\ and details
can be found in \cite{jair2013}.

When computing $sol_M(\varphi(A))$ for a given degree matrix we can no
longer apply the symmetry breaking Constraint~(\ref{eq:symbreak}) as
it might constrain the rows of $A$ in a way that contradicts the
constraint $\alpha(A)=M$ in the right side of
Equation~(\ref{eq:solM}). However, we can refine
Constraint~(\ref{eq:symbreak}, to break symmetries on the rows of $A$
only when the corresponding rows in $M$ are equal. Then $M$ can be
viewed as inducing an ordered partition of $A$ and
Constraint~(\ref{eq:sbdm}) is, in the terminology
of~\cite{DBLP:conf/ijcai/CodishMPS13}, a partitioned lexicographic
symmetry break.
In the following, $M_i$ and $M_j$ denote the $i^{th}$ and $j^{th}$
rows of matrix $M$.
%
%
\begin{equation}\label{eq:sbdm}
  \SB^*_\ell(A,M) = 
       \bigwedge_{i<j} \left(\begin{array}{l}
         \big(M_i=M_j\Rightarrow A_i\preceq_{\{i,j\}} A_j\big)
      \end{array}\right)
  \end{equation}
  The following refines Equation~(\ref{eq:solM}) introducing  the
  symmetry breaking predicate.
\begin{equation}
  \label{eq:scenario1}
  sol_M(\varphi(A)) = sol(\varphi(A)\wedge (\alpha(A){=}M) \wedge\SB^*_\ell(A,M))
\end{equation}

To justify that Equations~(\ref{eq:solM}) and~(\ref{eq:scenario1})
both compute $sol_M(\varphi(A))$, modulo weak isomorphism, we must show that
if $\SB^*_\ell(A,M)$ excludes a solution then there is another weakly
isomorphic solution that is not excluded. 

\begin{theorem}[\textbf{correctness of $\SB^{*}_\ell(A,M)$}]
\label{thm:sbl_star}
  Let $A$ be an adjacency matrix with $\alpha(A) = M$. Then, there
  exists $A'\approx A$ such that $\alpha(A')=M$ and 
  $\SB^{*}_\ell(A',M)$ holds.
\end{theorem}


\section{Computing Degree Matrices for $R(3,3,3;13)$}
\label{sec:33313}

This section describes how we compute a set of degree matrices
that approximate those of the solutions of instance
$\varphi_{(3,3,3;13)}(A)$ of Constraint~(\ref{constraint:coloring}). We
apply a strategy mixing SAT solving with brute-force
enumeration as follows. The computation of the degree matrices is
summarized in Table~\ref{tab:333_computeDMs}.
In the first step, we compute bounds on the degrees of the nodes in
any $R(3,3,3;13)$ coloring. 

\begin{lemma}\label{lemma:db}
  Let $A$ be an $R(3,3,3;13)$ coloring then for every vertex $x$ in $A$,
  and color $c\in\{1,2,3\}$, $2\leq deg_{c}(x)\leq 5$.
\end{lemma}

\begin{proof}
  By solving instance $\varphi_{(3,3,3;13)}(A)$ of
  Constraint~(\ref{constraint:coloring})
  %
  seeking a graph with some degree less than 2 or greater than 5.  The
  CNF encoding is of size 13{,}672 clauses with 2{,}748 Boolean
  variables and takes under 15 seconds to solve and yields an UNSAT
  result which implies that such a graph does not exist.
  %
\end{proof}

In the second step, we enumerate the degree sequences with values
within the bounds specified by Lemma~\ref{lemma:db}. Recall that the
degree sequence of an undirected graph is the non-increasing sequence
of its vertex degrees. Not every non-increasing sequence of integers
corresponds to a degree sequence. A sequence that corresponds to a
degree sequence is said to be graphical. The number of degree
sequences of graphs with 13 vertices is 836{,}315 (see Sequence number
\texttt{A004251} of The On-Line Encyclopedia of Integer Sequences
published electronically at \url{http://oeis.org}). However, when the
degrees are bound by Lemma~\ref{lemma:db} there are only 280.

\begin{lemma}\label{lemma:ds}
  There are 280 degree sequences with values between $2$ and $5$.
\end{lemma}

\begin{proof}
  Straightforward enumeration using the algorithm of Erd{\"{o}}s and
  Gallai~\cite{ErdosGallai1960}.
\end{proof}

In the third step, we test the 280 degree sequences identified
by Lemma~\ref{lemma:ds} to determine which of them might occur as the
left column in a degree matrix.

\begin{lemma}\label{lemma:ds2}
  Let $A$ be a $R(3,3,3;13)$ coloring and let $M=\alpha(A)$. Then, (a)
  the left column of $M$ is one of the 280 degree sequences identified
  in Lemma~\ref{lemma:ds}; and (b) there are only 80 degree sequences
  from the 280 which are the left column of $\alpha(A)$ for some
  coloring $A$ in $R(3,3,3;13)$.
\end{lemma}

\begin{proof}
  By solving instance $\varphi_{(3,3,3;13)}(A)$ of
  Constraint~(\ref{constraint:coloring}). For each degree sequence
  from Lemma~\ref{lemma:ds}, seeking a solution with that degree
  sequence in the first color.  This involves 280 instances with
  average CNF size: 10861 clauses and 2215 Boolean variables. The
  total solving time is 375.76 hours and the hardest instance required
  about 50 hours. Exactly 80 of these instances were satisfiable.
\end{proof}

In the fourth step we extend the 80 degree sequences identified in
Lemma~\ref{lemma:ds2} to obtain all possible degree matrices.

\begin{lemma}\label{lemma:dm}
  Given the 80 degree sequences identified in Lemma~\ref{lemma:ds2} as
  potential left columns of a degree matrix, there are 11{,}933
  possible degree matrices.
\end{lemma}
\begin{proof}
  By enumeration. For a degree matrix: the rows and columns are lex
  sorted, the rows must sum to 12, and the columns must be graphical
  (when sorted).  We enumerate all such degree matrices and then
  select their smallest representatives under permutations of rows and
  columns. The computation requires a few seconds.
\end{proof}

In the fifth step, we test  the 11{,}933 degree matrices
identified by Lemma~\ref{lemma:dm} to determine which of them 
are the abstraction of some $R(3,3,3;13)$ coloring. 

\begin{lemma}\label{lemma:dm2}
  From the 11{,}933 degree matrices identified in
  Lemma~\ref{lemma:dm}, 999 are $\alpha(A)$ for a coloring $A$ in
  $\RR(3,3,3;13)$. 
\end{lemma}

\begin{proof}
  By solving instance $\varphi_{(3,3,3;13)}(A)$ of
  Constraint~(\ref{constraint:coloring}) together with a given degree
  matrix to test if it is satisfiable.  This involves 11{,}933
  instances with average CNF size: 7632 clauses and 1520 Boolean
  variables. The total solving time is 126.55 hours and the hardest
  instance required 0.88 hours.
\end{proof}

\begin{table}[t]
\centering\scriptsize
\begin{tabular}{ |c|l|c|c|c|}
\hline
Step &  \multicolumn{1}{|c|}{Notes}&  
        \multicolumn{1}{|c|}{ComputationTimes} & 
        \multicolumn{2}{|c|}{CNF Size}\\
\hline\hline
\multirow{2}{*}{1}
  & compute degree bounds (Lemma~\ref{lemma:db})  & 
        \multirow{2}{*}{12.52 sec.} & \#Vars & \#Clauses\\
        \cline{4-5}
  & (1 instance, unsat)    &  ~    & \hfill 2748  &\hfill 13672   \\
\hline
\multirow{1}{*}{2}
  & enumerate 280 possible degree sequences (Lemma~\ref{lemma:ds})     & 
         \multicolumn{3}{|c|}{Prolog, fast (seconds)} \\
\hline

\multirow{2}{*}{3}
  & test degree sequences (Lemma~\ref{lemma:ds2}) & 16.32 hrs. & \#Vars & \#Clauses\\
                           \cline{4-5}
  & (280 instances: 200 unsat, 80 sat)    &  hardest: 1.34 hrs    & \hfill 1215 (avg)  &\hfill 7729(avg)   \\
  
\hline

{4}
  & enumerate 11{,}933 degree matrices (Lemma~\ref{lemma:dm})   & \multicolumn{3}{|c|}
                                                {Prolog, fast
                                                  (seconds)} \\
\hline
\multirow{2}{*}{5}
  & test degree matrices (Lemma~\ref{lemma:dm2}) & 126.55 hrs. & \#Vars & \#Clauses\\
                           \cline{4-5}
  & (11{,}933 instances: 10{,}934 unsat, 999 sat)    &  hardest: 0.88 hrs.    & \hfill 1520 (avg)  &\hfill 7632 (avg)   \\
\hline
\hline

\end{tabular}
\caption{Computing the degree matrices for $\RR(3,3,3;13)$ step by step.}
\label{tab:333_computeDMs}
\end{table}

\section{Computing $\RR(3,3,3;13)$ from Degree Matrices}
\label{sec:33313b}

We describe the computation of the set $\RR(3,3,3;13)$ starting from
the 999 degree matrices identified in
Lemma~\ref{lemma:dm2}. Table~\ref{tab:333_times} summarizes the two
step experiment.

\begin{table}[h]
\centering\scriptsize
\begin{tabular}{ |c|l|c|}
\hline
Step&  \multicolumn{1}{|c|}{Notes}&  \multicolumn{1}{|c|} {Computation Times} \\ 
\hline\hline
\multirow{2}{*}{1}
  & compute all $(3,3,3;13)$  Ramsey colorings per 
                           & total:~~\hfill 136.31 hr. \\
  & degree  matrix  (999 instances,  129{,}188 solutions)  
                           & hardest:\hfill 4.3 hr.\\
\hline
{2}
  & reduce modulo $\approx$  (78{,}892 solutions) 
                           & \multicolumn{1}{|c|}
                                      {\texttt{nauty}, fast (minutes)} \\
\hline
\hline
\end{tabular}
\caption{Computing  $\RR(3,3,3;13)$ step by step.}
\label{tab:333_times}
\end{table}


\vspace{-3mm}\paragraph{\bf step 1:} For each degree matrix we
compute, using a SAT solver, all corresponding solutions of
Equation~(\ref{eq:scenario1}), where
$\varphi(A)=\varphi_{(3,3,3;13)}(A)$ of
Constraint~(\ref{constraint:coloring}) and $M$ is one of the 999
degree matrices identified in (Lemma~\ref{lemma:dm2}).  This generates
in total 129{,}188 $(3,3,3;13)$ Ramsey colorings.
Table~\ref{tab:333_times} details the total solving time for these
instances and the solving times for the hardest instance for each SAT
solver. The largest number of graphs generated by a single instance is
3720.

\vspace{-3mm}\paragraph{\bf step 2:}
The 129{,}188 $(3,3,3;13)$ colorings from step~1 are reduced modulo
weak-isomorphism using \texttt{nauty}~\cite{nauty}. This
process results in a set with 78{,}892 graphs.

We note that recently, the set $\RR(3,3,3;13)$ has also been computed
independently by Stanislaw Radziszowski, and independently by Richard
Kramer and Ivan Livinsky \cite{stas:personalcommunication}.

\section{There is no $\tuple{13,8,8}$ Regular 
             $(4,3,3;30)$ Coloring}\label{sec:433_30}

In order to prove that there is no $\tuple{13,8,8}$ regular
$(4,3,3;30)$ coloring using the embedding approach of
Section~\ref{sec:embed}, we need to check that $78{,}892\times 3\times
3 = 710{,}028$ corresponding instances are unsatisfiable. These
correspond to the elements in the cross product of $\RR(3,3,3;13)$,
$\RR(4,2,3;8)$ and $\RR(4,3,2)$.

\newcommand{\ta}{\mathtt{A}} 
\newcommand{\tb}{\mathtt{B}} 
\newcommand{\tc}{\mathtt{C}} 
\begin{figure}
  \centering
$\left\{                             
\fbox{$\begin{scriptsize}\begin{smallmatrix}
0 & 1 & 1 & 1 & 3 & 3 & 3 & 3 \\
1 & 0 & 3 & 3 & 1 & 1 & 3 & 3 \\
1 & 3 & 0 & 3 & 1 & 3 & 1 & 3 \\
1 & 3 & 3 & 0 & 3 & 3 & 1 & 1 \\
3 & 1 & 1 & 3 & 0 & 3 & 3 & 1 \\
3 & 1 & 3 & 3 & 3 & 0 & 1 & 1 \\
3 & 3 & 1 & 1 & 3 & 1 & 0 & 3 \\
3 & 3 & 3 & 1 & 1 & 1 & 3 & 0
\end{smallmatrix}\end{scriptsize}$},
\fbox{$\begin{scriptsize}\begin{smallmatrix}
0 & 1 & 1 & 1 & 3 & 3 & 3 & 3 \\
1 & 0 & 3 & 3 & 1 & 3 & 3 & 3 \\
1 & 3 & 0 & 3 & 3 & 1 & 1 & 3 \\
1 & 3 & 3 & 0 & 3 & 1 & 3 & 1 \\
3 & 1 & 3 & 3 & 0 & 1 & 1 & 3 \\
3 & 3 & 1 & 1 & 1 & 0 & 3 & 3 \\
3 & 3 & 1 & 3 & 1 & 3 & 0 & 1 \\
3 & 3 & 3 & 1 & 3 & 3 & 1 & 0
\end{smallmatrix}\end{scriptsize}$},
\fbox{$\begin{scriptsize}\begin{smallmatrix}
0 & 1 & 1 & 1 & 3 & 3 & 3 & 3 \\
1 & 0 & 3 & 3 & 1 & 3 & 3 & 3 \\
1 & 3 & 0 & 3 & 3 & 1 & 1 & 3 \\
1 & 3 & 3 & 0 & 3 & 1 & 3 & 1 \\
3 & 1 & 3 & 3 & 0 & 1 & 3 & 3 \\
3 & 3 & 1 & 1 & 1 & 0 & 3 & 3 \\
3 & 3 & 1 & 3 & 3 & 3 & 0 & 1 \\
3 & 3 & 3 & 1 & 3 & 3 & 1 & 0
\end{smallmatrix}\end{scriptsize}$}\right\}
\subseteq
\left\{
    \fbox{$\begin{scriptsize}\begin{smallmatrix}
    0 & 1 & 1 & 1 & 3 & 3 & 3 & 3 \\
    1 & 0 & 3 & 3 & 1 & \ta & 3 & 3 \\
    1 & 3 & 0 & 3 & \ta & \tb & 1 & 3 \\
    1 & 3 & 3 & 0 & 3 & \tb & \ta & 1 \\
    3 & 1 & \ta & 3 & 0 & \tb & \tc & \ta \\
    3 & \ta & \tb & \tb & \tb & 0 & \ta & \ta \\
    3 & 3 & 1 & \ta & \tc & \ta & 0 & \tb \\
    3 & 3 & 3 & 1 & \ta & \ta & \tb & 0 \\
    \end{smallmatrix}\end{scriptsize}$}
\left| \begin{scriptsize}\begin{array}{l} 
         {\tiny \ta,\tb,\tc\in\{1,3\}} \\ 
         \ta\neq \tb
       \end{array}\end{scriptsize}
\right.\right\}$

\caption{Approximating the three (4,2,3;8) colorings by a single
  matrix with constraints.}
\label{figsubsumer}  
\end{figure}

To decrease the number of instances by a factor of $9$, we approximate
the three $(4,2,3;8)$ colorings by a single description as
demonstrated in Figure~\ref{figsubsumer}. The constrained matrix on
the right has four solutions which include the three $(4,2,3;8)$
colorings on the left. We apply a similar approach for the
$(4,3,2;8)$ colorings.   So, in fact we have a total of
only $78{,}892$ embedding instances to consider.

In addition to the constraints in Figure~\ref{fig:gcp}, we add
constraints to specify that each row of the adjacency matrix has the
prescribed number of edges in each color (13, 8 and 8). By application
of a SAT solver, we have determined all
\begin{wraptable}[13]{r}{5cm}\vspace{-7mm}
\centering
{\scriptsize\begin{tabular}{|r|r|r|}
\hline
time (hrs)    & \# instances  & \% instances ($\Delta$) \\ 
\hline \hline
10        & 56,363   & 71.443 \% \\ \hline
20        & 65,914   & 12.106 \% \\ \hline
100       & 77,263   & 14.385 \% \\ \hline
500       & 78,791   &  1.937 \% \\ \hline
1000      & 78,869   &  0.099 \% \\ \hline
1500      & 78,886    &  0.022 \% \\ \hline
2000      & 78,890    &  0.005 \% \\ \hline
2400      & 78,892    &  0.003 \% \\ \hline
\end{tabular}}
\caption{Time required per instance for proof that there are no
  $(4,3,3;30)$ colorings with degrees $(13,8,8)$ }
\label{hpi}
\end{wraptable}
$78{,}892$ instances to be unsatisfiable. The average size of an
instance is 36{,}259 clauses with 5187 variables. The total solving
time is 128.31 years (running in parallel on 456 threads).  The
average solving time is 14 hours while the median is 4 hours. Only 797
instances took more than one week to solve. The worst-case solving
time is 96.36 days. The two hardest instances are detailed in
Appendix~\ref{apdx:hardest}.
Table~\ref{hpi} specifies, in the second column, the total number of
instances that can be shown unsatisfiable within the time specified in
the first column. The third column indicates the increment in
percentage (within 10 hours we solve 71.46\%, within 20 hours we solve
an additional 12.11\%, etc). The last rows in the table indicate that
there are 4 instances which require between 1500 and 2000 hours of
computation, and 2 that require between 2000 and 2400 hours.
%


\section{Conclusion}
\label{sec:conclude}



We have applied SAT solving techniques together with a methodology
using abstraction and symmetry breaking to construct a computational
proof that the Ramsey number $R(4,3,3)=30$.
Our strategy is based on the search for a $(4,3,3;30)$ Ramsey
coloring, which we show does not exist. This implies that
$R(4,3,3)\leq 30$ and hence, because of known bounds, that
$R(4,3,3) = 30$.



The precise value $R(4,3,3)$ has remained unknown for almost 50 years.
We have applied a methodology involoving SAT solving, abstraction, and
symmetry to compute $R(4,3,3)=30$. We expect this methodology to apply
to a range of other hard graph coloring problems.

The question of whether a computational proof constitutes a {\it
  proper} proof is a controversial one. Most famously the issue caused
much heated debate after publication of the computer proof of the Four
Color Theorem \cite{appel76}.  It is straightforward to justify an
existence proof (i.e. a {\it SAT} result), as it is easy to verify
that the witness produced satisfies the desired properties. Justifying
an {\it UNSAT} result is more difficult. If nothing else, we are certainly
required to add the proviso that our results are based on the
assumption of a lack of bugs in the entire tool chain (constraint
solver, SAT solver, C-compiler etc.) used to obtain them.

Most modern SAT solvers, support the option to generate a proof
certificate for UNSAT instances (see e.g.~\cite{HeuleHW14}), in the
DRAT format~\cite{WetzlerHH14}, which can then be checked by a Theorem
prover. This might be useful to prove the lack of bugs originating
from the SAT solver but does not offer any guarantee concerning bugs
in the generation of the CNF. Moreover, the DRAT certificates for an
application like that described in this paper are expected to be of
unmanageable size.

Our proofs are based on two main ``computer programs''. The first was
applied to compute the set $\RR(3,3,3;13)$ with its $78{,}892$ Ramsey
colorings. The fact that at least two other groups of researchers
(Stanislaw Radziszowski, and independently Richard Kramer and Ivan
Livinsky) report having computed this set and
quote~\cite{stas:personalcommunication} the same number of elements is
reassuring.
The second program, was applied to complete partially instantiated
adjacency matrices, embedding smaller Ramsey colorings, to determine
if they can be extended to Ramsey colorings. This program was applied
to show the non-existence of a $(4,3,3;30)$ Ramsey coloring.
Here we gain confidence from the fact that the same program does find
Ramsey colorings when they are known to exist. For example, the
$(4,3,3;29)$ coloring depicted as Figure~\ref{embed_12_8_8}.

All of the software used to obtain our results is publicly
available, as well as the individual constraint models and their
corresponding encodings to CNF. For details, see the appendix.



\subsection*{Acknowledgments}
We thank Stanislaw Radziszowski for his guidance and comments which
helped improve the presentation of this paper. In particular
Stanislaw proposed to show that our technique is able to find the
$(4,3,3;29)$ coloring depicted as Figure~\ref{embed_12_8_8}.

\newpage
\appendix

\section{Selected Proofs}\label{proofs}

\textbf{Lemma}~~\ref{lemma:closed}. ~~~
[\textbf{$\RR(r_1,r_2,\ldots,r_k;n)$ is closed under
    $\approx$}]

\smallskip\noindent
 Let $(G,{\kappa_1})$ and $(H,{\kappa_2})$ be graph colorings
  in $k$ colors such that $(G,\kappa_1) \approx_{\pi,\sigma}
  (H,\kappa_2)$. Then,
  \[
  (G,\kappa_1) \in \RR(r_1,r_2,\ldots,r_k;n) \iff (H,\kappa_2) \in
  \RR(\sigma(r_1),\sigma(r_2),\ldots,\sigma(r_k);n)
  \]

\medskip
\begin{proof}[of Lemma~\ref{lemma:closed}]
  Assume that $(G,\kappa_1) \in \RR(r_1,r_2,\ldots,r_k;n)$ and
  in contradiction that $(H,\kappa_2) \notin
  \RR(\sigma(r_1),\sigma(r_2),\ldots,\sigma(r_k);n)$. Let $R$ denote a
  monochromatic clique of size $r_s$ in $H$ and $R^{-1}$ the inverse
  of $R$ in $G$.
  From Definition~\ref{def:weak_iso},
  $(u,v) \in R \iff (\pi^{-1}(u), \pi^{-1}(v))\in R^{-1}$ and
  $\kappa_2(u,v) = \sigma^{-1}(\kappa_1(u,v))$.  Consequently $R^{-1}$ is a
  monochromatic clique of size $r_s$ in $(G,\kappa_1)$ in contradiction
  to $(G,\kappa_1)$ $\in$ $\RR(r_1,r_2,\ldots,r_k;n)$. 
\end{proof}

\noindent
\textbf{Theorem}~~\ref{thm:sbl_star}. ~~~
[\textbf{correctness of $\SB^{*}_\ell(A,M)$}]
\smallskip\noindent
  Let $A$ be an adjacency matrix with $\alpha(A) = M$. Then, there
  exists $A'\approx A$ such that $\alpha(A')=M$ and 
  $\SB^{*}_\ell(A',M)$ holds.

\smallskip\noindent
\begin{proof}[of Theorem~\ref{thm:sbl_star}]
  Let $C=\sset{A'}{A'\approx A \wedge \alpha(A')=M}$.
  Obviously $C\neq \emptyset$ because $A\in C$ and therefore
  there exists a $A_{min}=min_{\preceq} C$. 
  Therefore, $A_{min} \preceq A'$ for all $A' \in C$. 
  Now we can view $M$ as inducing an ordered partion on $A$: vertices
  $u$ and $v$ are in the same component if and only if the
  corresponding rows of $M$ are equal.
  Relying on Theorem 4 from \cite{DBLP:conf/ijcai/CodishMPS13}, 
  we conclude that $\SB^{*}_\ell(A_{min},M)$ holds.
\end{proof}

\section{The Two Hardest Instances}\label{apdx:hardest}
\newcommand{\td}{\mathtt{D}} \newcommand{\te}{\mathtt{E}}
\newcommand{\tf}{\mathtt{F}} The following partial adjacency matrices
are the two hardest instances described in Section~\ref{sec:433_30},
from the total 78{,}892.  Both include the constraints:
$\ta,\tb,\tc,\in\{1,3\}$, $\td,\te,\tf\in\{1,2\}$, $\ta\neq \tb$,
$\td\neq \te$.  The corresponding CNF representations consist in 5204
Boolean variables (each), 36{,}626 clauses for the left instance and
36{,}730 for the right instance. SAT solving times to show these
instances UNSAT are 8{,}325{,}246 seconds for the left instance and
7{,}947{,}257 for the right.


\newcommand{\undE}{\_&\_&\_&\_&\_&\_&\_&\_&\_&\_&\_&\_&\_&\_&\_&\_}
\newcommand{\undF}{\_&\_&\_&\_&\_&\_&\_&\_&\_&\_&\_&\_&\_}
\newcommand{\undG}{\_&\_&\_&\_&\_&\_&\_&\_}
\newcommand{\undH}{\_&\_&\_&\_&\_&\_&\_&\_&\_&\_&\_&\_&\_&\_&\_&\_&\_&\_&\_&\_&\_}

\medskip
\noindent
\resizebox{.49\linewidth}{!}{$\left[                                  
\begin{smallmatrix}
0&1&1&1&1&1&1&1&1&1&1&1&1&1&2&2&2&2&2&2&2&2&3&3&3&3&3&3&3&3&\\
1& 0&1&1&2&2&3&3&1&1&2&3&3&1&\undE \\
1& 1&0&2&1&3&1&2&2&3&3&1&2&3&\undE \\
1& 1&2&0&3&2&1&1&3&3&1&2&3&2&\undE \\
1& 2&1&3&0&1&2&3&1&1&3&3&2&2&\undE \\
1& 2&3&2&1&0&1&2&3&2&1&3&3&1&\undE \\
1& 3&1&1&2&1&0&2&2&3&3&2&1&3&\undE \\
1& 3&2&1&3&2&2&0&3&3&2&1&1&1&\undE \\
1& 1&2&3&1&3&2&3&0&2&3&1&1&3&\undE \\
1& 1&3&3&1&2&3&3&2&0&1&3&1&2&\undE \\
1& 2&3&1&3&1&3&2&3&1&0&1&3&2&\undE \\
1& 3&1&2&3&3&2&1&1&3&1&0&2&3&\undE \\
1& 3&2&3&2&3&1&1&1&1&3&2&0&3&\undE \\
1& 1&3&2&2&1&3&1&3&2&2&3&3&0&\undE \\
2& \undF&0 & 1 & 1 & 1 & 3 & 3 & 3 & 3           &\undG \\
2& \undF&1 & 0 & 3 & 3 & 1 & \ta & 3 & 3         &\undG \\
2& \undF&1 & 3 & 0 & 3 & \ta & \tb & 1 & 3       &\undG \\
2& \undF&1 & 3 & 3 & 0 & 3 & \tb & \ta & 1       &\undG \\
2& \undF&3 & 1 & \ta & 3 & 0 & \tb & \tc & \ta   &\undG \\
2& \undF&3 & \ta & \tb & \tb & \tb & 0 & \ta &\ta&\undG \\
2& \undF&3 & 3 & 1 & \ta & \tc & \ta & 0 & \tb   &\undG \\
2& \undF&3 & 3 & 3 & 1 & \ta & \ta & \tb & 0     &\undG \\
3& \undH&0 & 1 & 1 & 1 & 2 & 2 & 2 & 2            \\
3& \undH&1 & 0 & 2 & 2 & 1 & \td & 2 & 2          \\
3& \undH&1 & 2 & 0 & 2 & \td & \te & 1 & 2        \\
3& \undH&1 & 2 & 2 & 0 & 2 & \te & \td & 1        \\
3& \undH&2 & 1 & \td & 2 & 0 & \te & \tf & \td    \\
3& \undH&2 & \td & \te & \te & \te & 0 & \td &\td \\
3& \undH&2 & 2 & 1 & \td & \tf & \td & 0 & \te    \\
3& \undH&2 & 2 & 2 & 1 & \td & \td & \te & 0      \\
\end{smallmatrix}\right]$}
\hfill
\resizebox{.49\linewidth}{!}{$\left[                                  
\begin{smallmatrix}
0&1&1&1&1&1&1&1&1&1&1&1&1&1&2&2&2&2&2&2&2&2&3&3&3&3&3&3&3&3&\\
 1& 0&1&1&2&1&2&3&3&3&2&1&2&3&\undE \\
 1& 1&0&3&1&3&2&1&1&3&2&2&3&2&\undE \\
 1& 1&3&0&3&2&1&2&3&1&2&3&2&1&\undE \\
 1& 2&1&3&0&1&1&3&2&3&1&2&3&2&\undE \\
 1& 1&3&2&1&0&2&3&3&1&3&2&1&3&\undE \\
 1& 2&2&1&1&2&0&3&3&3&3&1&1&3&\undE \\
 1& 3&1&2&3&3&3&0&2&1&1&3&1&2&\undE \\
 1& 3&1&3&2&3&3&2&0&1&2&1&3&1&\undE \\
 1& 3&3&1&3&1&3&1&1&0&2&3&2&2&\undE \\
 1& 2&2&2&1&3&3&1&2&2&0&3&3&1&\undE \\
 1& 1&2&3&2&2&1&3&1&3&3&0&2&3&\undE \\
 1& 2&3&2&3&1&1&1&3&2&3&2&0&3&\undE \\
 1& 3&2&1&2&3&3&2&1&2&1&3&3&0&\undE \\
2& \undF&0 & 1 & 1 & 1 & 3 & 3 & 3 & 3           &\undG \\
2& \undF&1 & 0 & 3 & 3 & 1 & \ta & 3 & 3         &\undG \\
2& \undF&1 & 3 & 0 & 3 & \ta & \tb & 1 & 3       &\undG \\
2& \undF&1 & 3 & 3 & 0 & 3 & \tb & \ta & 1       &\undG \\
2& \undF&3 & 1 & \ta & 3 & 0 & \tb & \tc & \ta   &\undG \\
2& \undF&3 & \ta & \tb & \tb & \tb & 0 & \ta &\ta&\undG \\
2& \undF&3 & 3 & 1 & \ta & \tc & \ta & 0 & \tb   &\undG \\
2& \undF&3 & 3 & 3 & 1 & \ta & \ta & \tb & 0     &\undG \\
3& \undH&0 & 1 & 1 & 1 & 2 & 2 & 2 & 2            \\
3& \undH&1 & 0 & 2 & 2 & 1 & \td & 2 & 2          \\
3& \undH&1 & 2 & 0 & 2 & \td & \te & 1 & 2        \\
3& \undH&1 & 2 & 2 & 0 & 2 & \te & \td & 1        \\
3& \undH&2 & 1 & \td & 2 & 0 & \te & \tf & \td    \\
3& \undH&2 & \td & \te & \te & \te & 0 & \td &\td \\
3& \undH&2 & 2 & 1 & \td & \tf & \td & 0 & \te    \\
3& \undH&2 & 2 & 2 & 1 & \td & \td & \te & 0      \\
\end{smallmatrix}\right]$}
%


\section{Making the Instances Available}

The statistics from the proof that $R(4,3,3)=30$ are
available from the domain:
\begin{quote}
  \url{http://cs.bgu.ac.il/~mcodish/Benchmarks/Ramsey334}.
\end{quote}
Additionally, we have made a small sample (30) of the instances
available. Here we provide instances with the degrees $\tuple{13,8,8}$
in the three colors. The selected instances represent the varying
hardness encountered during the search. 
The instances numbered $\{27765$, $39710$, $42988$, $36697$, $13422$,
$24578$, $69251$, $39651$, $43004$, $75280\}$ are the hardest, the
instances numbered $\{4157$, $55838$, $18727$, $43649$, $26725$,
$47522$, $9293$, $519$, $23526$, $29880\}$ are the median, and the
instances numbered $\{78857$, $78709$, $78623$, $78858$, $28426$,
$77522$, $45135$, $74735$, $75987$, $77387\}$ are the easiest.
A complete set of both the \bee\ models and the DIMACS CNF files are
available upon request.  Note however that they weight around 50GB
when zipped.

The files in \url{bee_models.zip} detail constraint models,
each one in a separate file. The file named
\texttt{r433\_30\_Instance\#.bee} contains a single
Prolog clause of the form
\begin{quote}
  \texttt{model(Instance\#,Map,ListOfConstraints) :- \{...details...\} .}
\end{quote}
where \texttt{Instance\#} is the instance number, \texttt{Map} is a
partially instantiated adjacency matrix associating the unknown
adjacency matrix cells with variable names, and
\texttt{ListOfConstraints} are the finite domain constraints defining
their values. The syntax is that of \bee, however the interested
reader can easily convert these to their favorite fininte domain
constraint language.  Note that the Boolean values $\true$ and
$\false$ are represented in \bee\ by the constants $1$ and $-1$.
Figure~\ref{fig:bee} details the \bee\ constraints which occur in the
above mentioned models. 

\begin{figure}[t]
  \centering
 \begin{tabular}{rlll}
\hline\hline
(1) &$\mathtt{new\_int(I,c_1,c_2)}$ && declare integer: $\mathtt{c_1\leq I\leq c_2}$\\
(2) &$\mathtt{bool\_array\_or([X_1,\ldots,X_n])}$ &&
          clause: $\mathtt{X_1 \vee X_2 \cdots \vee X_n}$\\
(3) &    $\mathtt{bool\_array\_sum\_eq([X_1,\ldots,X_n],~I)}$ &&
          Boolean cardinality: $\mathtt{(\Sigma ~X_i) = I}$\\
(4) &    $\mathtt{int\_eq\_reif(I_1,I_2,~X)}$ &&
          reified integer equality: $\mathtt{I_1 = I_2 \Leftrightarrow X}$\\
(5) &    $\mathtt{int\_neq(I_1,I_2)}$ &
          $\mathtt{}$&
          $\mathtt{I_1 \neq I_2}$\\
(6) &    $\mathtt{int\_gt(I_1,I_2)}$ &
          $\mathtt{}$&
          $\mathtt{I_1 > I_2}$\\
\hline\hline
\end{tabular}
  \caption{Selected \bee\ constraints}
  \label{fig:bee}
\end{figure}

The files in \url{cnf_models.zip} correspond to CNF encodings
for the constraint models.  Each instance is associated with two
files: \texttt{r433\_30\_instance\#.dimacs} and
\texttt{r433\_30\_instance\#.map}. These consist respectively in a
DIMACS file and a map file which associates the Booleans in the DIMACS
file with the integer variables in a corresponding partially
instantiated adjacency matrix. The map file specifies for each pair
$(i,j)$ of vertices a triplet $[B_1,B_2,B_3]$ of Boolean variables (or
values) specifying the presence of an edge in each of the three
colors. Each such $B_i$ is either the name of a DIMACS variable, if it
is greater than 1, or a truth value $1$ ($\true$), or $-1$ ($\false$).


\end{document}